\documentclass[fleqn,10pt]{wlscirep}
\usepackage[utf8]{inputenc}
\usepackage[T1]{fontenc}

\usepackage{amsmath}
\usepackage{amssymb}
\usepackage{amsthm}
\usepackage{amsfonts}

\theoremstyle{plain}
\newtheorem{theorem}{Theorem}[section]

\newtheorem{lemma}[theorem]{Lemma}
\newtheorem{corollary}[theorem]{Corollary}
\theoremstyle{definition}
\newtheorem{definition}[theorem]{Definition}

\theoremstyle{remark}

\usepackage{mathtools}
\usepackage{algorithm}
\usepackage[noend]{algpseudocode}
\usepackage{enumitem}
\usepackage{microtype}

\usepackage{multirow}
\usepackage{subcaption}

\title{On the Exact Algorithmic Extraction of Finite Tesselations Through Prime Extraction of Minimal 
Representative Forms}

\author{Sushish Baral$^1$, Paulo Garcia$^{2}$ and Warisa Sritriratanarak$^1*$}
\affil{$^1$Department of Computer Engineering, Chulalongkorn University, Bangkok, Thailand}
\affil{$^2$International School of Engineering, Chulalongkorn University, Bangkok, Thailand}
\affil{6773005021@student.chula.ac.th, paulo.g@chula.ac.th, warisa.s@chula.ac.th}
\affil{$^*$ Corresponding author: Warisa Sritriratanarak (warisa.s@chula.ac.th)}

\begin{abstract}
The identification of repeating patterns in discrete grids is rudimentary within symbolic reasoning, algorithm synthesis and structural optimization across diverse 
computational domains. Although statistical approaches targeting noisy data can approximately recognize patterns, symbolic analysis utilizing  deterministic extraction of periodic structures is underdeveloped. This paper aims to fill this gap by employing a hierarchical algorithm that discovers exact tessellations in finite planar grids, addressing the problem where multiple independent patterns may coexist within a hierarchical structure.
The proposed method utilizes composite discovery (dual inspection and breadth-first pruning) for identifying rectangular regions with internal repetition, normalization to a minimal representative form, and prime extraction (selective duplication and hierarchical memoization) to account for irregular dimensions and to achieve efficient computation time. We evaluate scalability on grid sizes from $2\times 2$ to $32 \times 32$, showing overlap detection on simple repeating tiles exhibits processing time under 1ms, while complex patterns which require exhaustive search and systematic exploration shows exponential growth. This algorithm provides deterministic behavior for exact, axis-aligned, rectangular tessellations,
addressing a critical gap in symbolic grid analysis techniques,  applicable to  puzzle solving reasoning tasks and identification of exact repeating structures in discrete symbolic domains. . 
\end{abstract}
\begin{document}

\flushbottom
\maketitle
%
%
\thispagestyle{empty}

\section{Introduction}

Tesselation \cite{10.1145/3597932} as an architectural and artistic feature has been documented for thousands of years \cite{chang2018application}. The mathematical study of its properties, particularly with regards to periodic tiling \cite{greenfeld2024counterexample} began in the late 19th century with the work of Fyodorov \cite{vzdimalova2020tessellation}, and aperiodic tiling was seminally explored by Penrose \cite{bursill1985penrose} (with the discovery of the first aperiodic monotile recently \cite{smith2023aperiodic}). Whilst a significant body of work on the mathematical properties and the generation of periodic tiles exist \cite{caceres2002aperiodic}, there is little work on the deterministic exact identification (i.e., \textit{extraction}) of tiling patterns. The field of pattern recognition \cite{7091016} (and, more recently, the sub-field of machine learning \cite{ngiam2010tiled}) has developed statistical extraction techniques \cite{valiente2004structural}, mostly aimed at tiling extraction from real-world fuzzy data (e.g., images \cite{pradella2008sat}); however, exact extraction as an algorithmic technique remains under-developed. This is fueled by two reasons: a) for most real-world applications, data is fuzzy, so statistical solutions are both necessary and sufficient; b) the problem is computationally complex, with several corner cases challenging algorithmic solutions.

\par In this paper, we explore such an algorithmic solution, and associated data structures, for the extraction of exact tiling patterns on planar grids \cite{grunbaum1987tilings}. Specifically, we address the problem of identification of tesselations \cite{phillips2014tessellation} of (finite) planar grids, where several different independent tesselations may occur on the same grid, and a tesselation may consist of hierarchical tiling patterns \cite{zhou2012hierarchical}: a more general problem than identifying strictly periodic planar tesselations. This problem arises in applications such as routing matrices \cite{10.1145/1290672.1290688} (where recurring tiles represent common routing patterns that lend themselves to optimization); path-planning \cite{10.1145/1383369.1383372} (where identifying a recurring tile allows the re-use of a previously memoized solution); causality inference in 2D worlds (where tiling patterns are created by or affect other structures; e.g., ARC-AGI \cite{chollet2024arc}); and others. 

 The work in this paper is motivated by the authors' work on the ARC-AGI challenge, where tesselation extraction plays a crucial role in several puzzles, and machine learning has thus far failed to perform satisfactorily. Specifically, this paper offers the following contributions:

\begin{itemize}
  \item We present an algorithm for exact tiling pattern extraction on 2D structures. The described algorithm can deterministically identify building-block tiles that satisfy a given tesselation. 
  \item We describe and evaluate the performance of optimization techniques (discovery of minimal representative forms and normalization) that prune the search space by exploiting the inherent symmetries of tesselations.
  \item We provide an open-source reference implementation\footnote{\url{https://github.com/paulo-chula/Exact-Tesselation-Extraction}} and evaluate its performance across representative examples.
\end{itemize}

The remainder of this paper is organized as follows. Section \ref{sec:background} provides some motivating examples for the need for the techniques described in this paper, as well as a formalization of the problem.
Paper organization: Section \ref{sec:hierarchical} describes the proposed algorithms and associated optimization techniques for search-space pruning. Section \ref{sec:experiments} describes the experimental setup used to evaluate the paper's contributions. Section \ref{sec:related} presents an overview of the state of the art in the field, and Section \ref{sec:conclusions} concludes this paper.

\section{Background}\label{sec:background}

\subsection{Motivating Example}



\begin{figure}[th!]
    \centering
    \begin{subfigure}[t]{0.6\textwidth}
        \includegraphics[width=\textwidth]{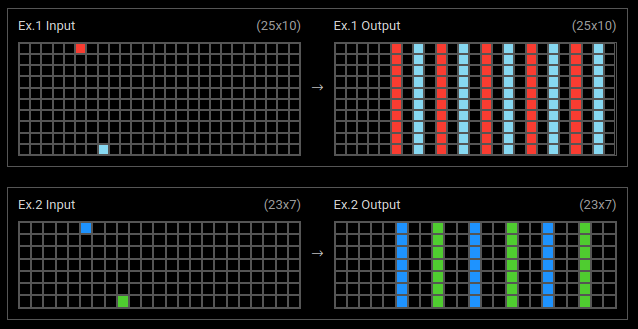}
    \end{subfigure}
    \begin{subfigure}[t]{0.9\textwidth}
        \centering
        \begin{subfigure}[t]{0.3\textwidth}
        \includegraphics[width=\textwidth]{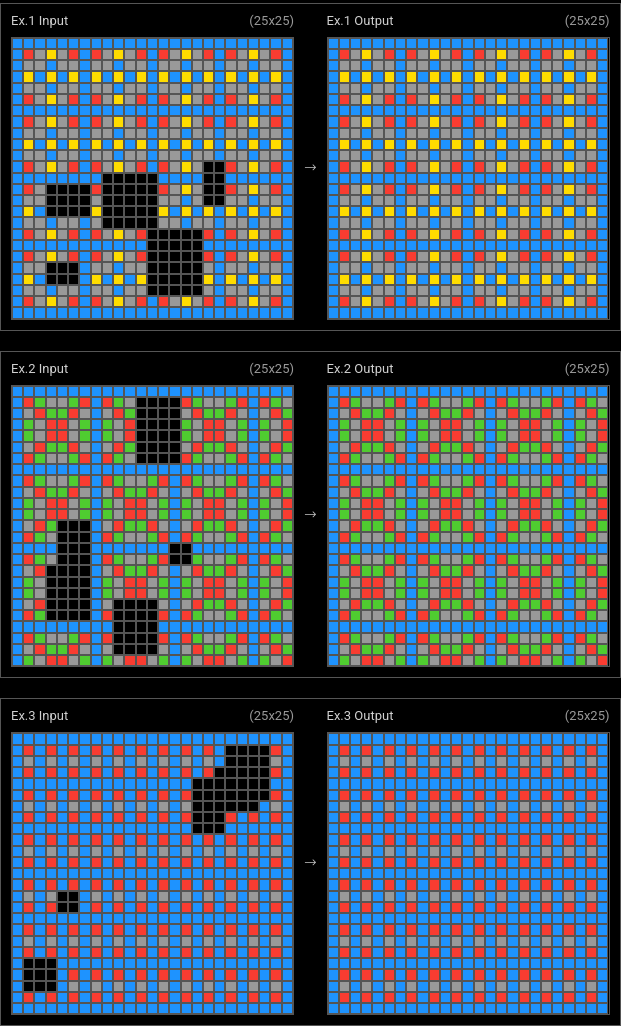}
        \end{subfigure}
        \quad
        \begin{subfigure}[t]{0.2\textwidth}
        \includegraphics[width=\textwidth]{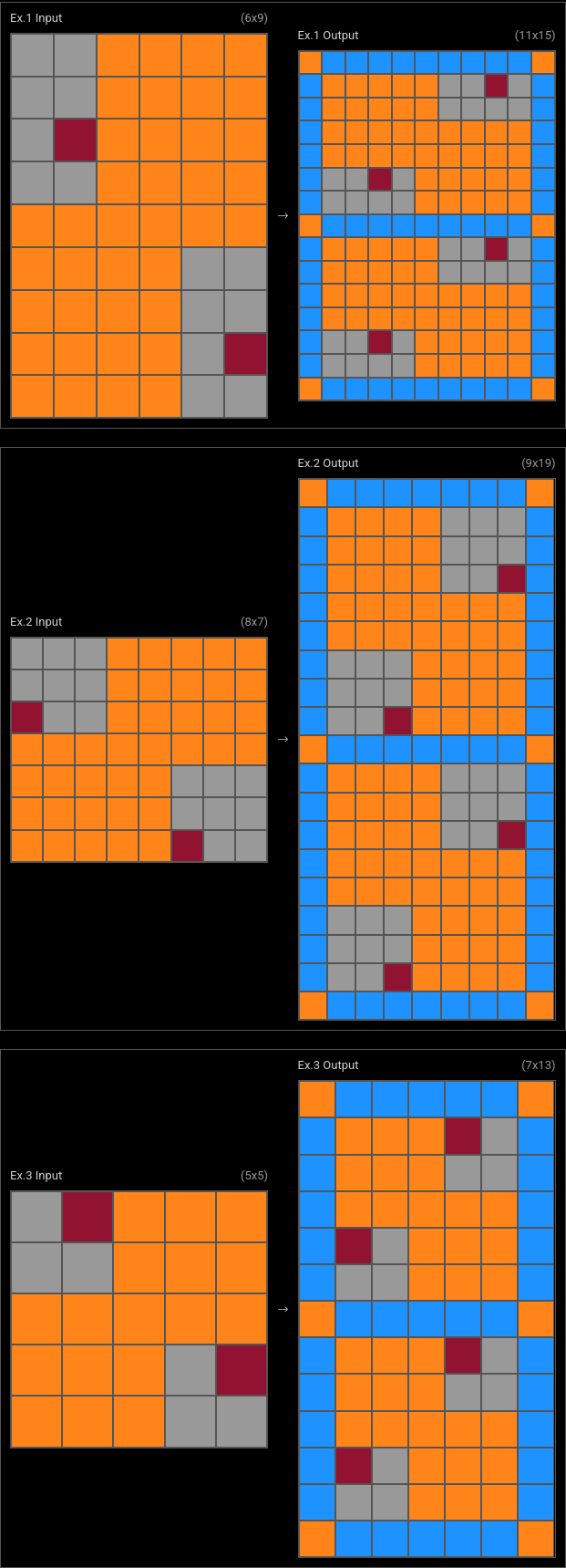}
        \end{subfigure}
        \quad
        \begin{subfigure}[t]{0.3\textwidth}
        \includegraphics[width=\textwidth]{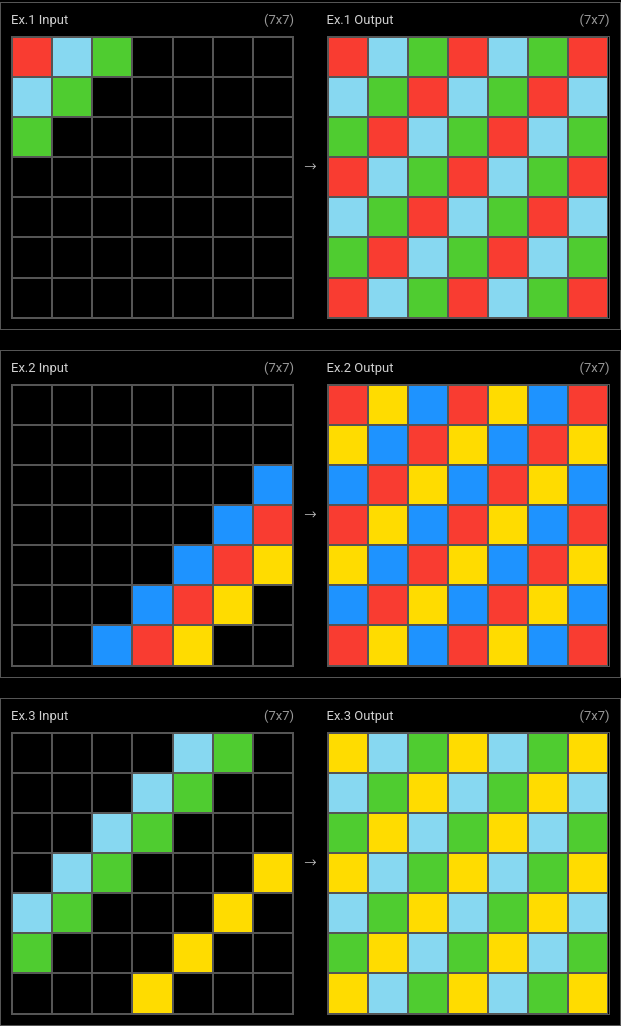}
        \end{subfigure}
    \end{subfigure}
    \caption{ARC tasks; reproduced from \url{https://arcprize.org/play?task=0a938d79}}
    \label{fig:arc}
\end{figure}

The ARC-AGI challenge \cite{chollet2024arc} consists of a set of tests, with training tests publicly available an evaluation tests kept private. In each test, a set of example input-output pairs (typically, 3-5 pairs) depicts a specific process: the goal is to infer that process and generate the output for a given input. Across training and evaluation tests, several different processes exist. Empirical evaluation has shown that human beings can infer all processes \cite{legris2024h}, but machine-learning solutions have not demonstrated great proficiency thus far \cite{pfister2025understanding} (primarily, due to the absence of large training data for the challenge), although Test-Time-Training has recently shown promise \cite{zhut5}.
\par \textit{Program synthesis}, i.e., the inference of a program that can generate outputs from respective inputs from example pairs is a promising solution \cite{ouellette2024towards}, particularly when applied within a Domain Specific Language (DSL \cite{mernik2005and}) that can efficiently express the semantics of the domain \cite{franzenllm}. Consider the example test depicted in at the top of Fig. \ref{fig:arc}, from the public training dataset. A solution can be described in natural language as:

\begin{itemize}
    \item "Extend points into lines, perpendicularly to the "walls" they are adjacent to."
    \item "In the direction of the empty half of the grid, replicate the lines in an alternating fashion, such that the distance between lines is constant."
\end{itemize} 

Implementing such a solution, is a DSL or general-purpose language, is also trivial: but, inferring such a solution computationally in the first place is not trivial, despite the ease with which humans can do so. One such way is to analyze input-output pairs and identify common occurrences. In this case, a solution must correctly identify that: (1) all output grids exhibit a tesselation; (2) there are constraints on the area occupied by that tesselation, and those constraints can be inferred from input grids; (3) there exists a sub-computation (extending points into lines) that can generate a valid tile for such output tesselations from input grids. 
\par We were surprised to find that no algorithm to exactly identify tesselations and corresponding tiles exists, leading to the work in this paper.

\subsection{Preliminaries}

A finite planar grid $\mathbf{G}$ is defined as an $n\times m$ matrix, $n, m \in \mathbb{N}$, where the value in each element $g_{i,j} \in \mathbb{N}, \forall i \in \{1:n\}, \forall j \in \{1:m\}$, can be thought of as isomorphic to the set of colors used to represent the grid visually. As the specific isomorphism is not relevant to the algorithmic aspects, throughout the remainder of this paper we will notate values in cells as natural numbers, and represent grids visually as matrices, without explicitly or consistently mapping numbers to colors, without loss of generality. We will use standard matrix notation, such that element $g_{1,1}$ is the top-left corner of a grid, and element $g_{n,m}$ is the bottom-right corner of a grid.

\begin{definition}
\label{def:box}
A \textit{rectangular bounding box} (notated $R$) in $\mathbf{G}$ encapsulates a rectangular sub-region of $\mathbf{G}$. A rectangular bounding box is a tuple $R = (x,y,l,h)$, where $x \in [1:n],y \in [1:m]$ are the coordinates of the top-left corner of the box, and $l \in [1:n],h \in [1,m], (l > 1) \cup (h > 1), ((x+l) \leq (n-1)) \cap ((y+h) \leq (m-1))$ are the length and height, respectively; where element $g_{i,j}$ is considered in the bounding box if $(x < i < (x+l - 1)) \cap(y < j < (y+h - 1))$.
\end{definition}

\begin{lemma}
\label{lem:num_boxes}
If $\mathbf{G}$ is of size $n\times m$, there are $\frac{nm((n+1)(m+1)-4)}{4}$ possible rectangular bounding boxes in $\mathbf{G}$.
\end{lemma}

\begin{proof} 
As a rectangle is defined by its projections on the $x,y$ axes, and a line of integer length $l$ admits $\frac{l(l+1)}{2}$ integer sub-projections on a parallel axis, the combinations of projections are given by $\frac{n(n+1)}{2} \times \frac{m(m+1)}{2} = \frac{n(n+1)m(m+1)}{4}$ This includes all $n\times m$ rectangles (squares) of size $1\times1$, which must be subtracted from the total number: thus, $\frac{n(n+1)m(m+1)}{4} - n\times m = \frac{nm((n+1)(m+1)-4)}{4}$.
\end{proof}

\begin{definition}
\label{def:tile}
A rectangular bounding box $R_u = (x_u,y_u,l,h)$ is a \textit{rectangular tile} in $\mathbf{G}$ if $\exists R_v = (x_v,y_v,l,h) \rightarrow ((x_v = x_u + l, y_v = y_u) \cup (x_v = x_u - l, y_v = y_u) \cup (x_v = x_u , y_v = y_u + h) \cup (x_v = x_u, y_v = y_u - h))  \cap (\forall (a \in [0,l],b \in [0,h]), g_{x_u + a,y_u+b} = g_{x_v+a,y_v+b}) $.
\end{definition}

\begin{corollary}
The definition of rectangular tile can be trivially expanded to include rotations (for square tiles) and symmetries, if needed.
\end{corollary}

I.e., we define tiles as rectangular bounding boxes in $\mathbf{G}$ which, when translated horizontally by $l$ or when translated vertically by  $h$, perfectly matches the values originally at the new position.

\begin{definition}
\label{def:tesselation}
A \textit{tesselation} $T$ of rectangular tiles $t$ is a continuous subset of $\mathbf{G}$ that can be drawn by starting with a tile $t$ and adding connected translations of $t$, where all horizontal translations are of distance $l$ and all vertical translations are of distance $h$. Furthermore, tesselations may be part of tiles (Fig. \ref{fig:tesselations}).
\end{definition}

\begin{figure}[ht]
\begin{center}
\centerline{\includegraphics[width=0.8\columnwidth]{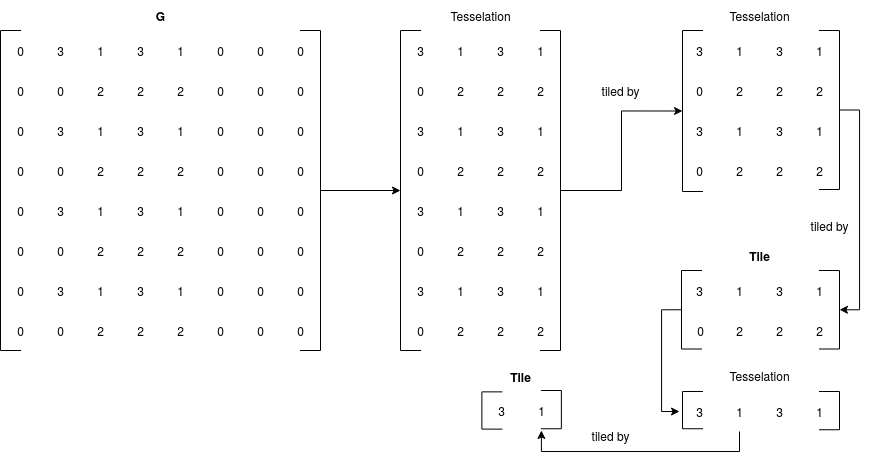}}
\caption{Example grid and identified tesselations, with connections to respective tiles.}
\label{fig:tesselations}
\end{center}
\end{figure}

Note that, in mathematics, it is more common to define (periodic) tiles as geometric shapes (typically represented visually as of a single color) that can be arranged such that they tile the plane. In contrast, we care about tiles as rectangular bounding boxes (which represent the spatial period of the geometric tiles, and thus are guaranteed to emerge from any sufficiently large tesselation), as these suffice to identify a large subset of all possible tesselations, and minimize the search space of all possible tiles.


\section{Hierarchical Composite and Prime Discovery}\label{sec:hierarchical}

\subsection{Definitions}

We define the following key concepts used throughout our approach:

\textbf{Composite}: A section within the area of matrix that exhibits internal overlap, meaning a 
row-wise/column-wise split yeilds on identical halves. For example, 
a $4 \times 6$ tile \texttt{1212/3434/1212/3434} is a composite because splitting it row-wise yields 
two identical $2 \times 6$ halves.

\textbf{Normalized Composite}: The composite after reduced into its minimal representative form obtain 
through iterative halving are normalized composites. A composite \texttt{12121212} ($1 \times 8$) normalizes 
to \texttt{1212} ($1 \times 4$) and then \texttt{1212} ($1 \times 2$) after successive column reductions.

\textbf{Prime}: Primes are the basic tiles with which the whole surface can be built. Its the atomic 
tile that cannot be further decomposed. Primes are the building blocks used to reconstruct the original 
composite through tiling.

\textbf{Overlap}: A process of comparing tiles by placing one on top of another which is performed after the 
tiles are split (row wise or column wise) into two halves. Overlap detection is performed with or without 
duplication depending on the phase.

\textbf{Pruning}: The operation of removing one row or column from an edge of a rectangular 
tile (top, bottom, left, or right), generating a smaller child tile. Pruning enables exhaustive 
exploration of all sub-rectangles via breadth-first search.

\textbf{Branch Cut}: A termination of exploration in the BFS tree when overlap is detected in a certain node, 
This prevent exploration of redundant sub-trees, since all descendants are guaranteed to be contained 
within the discovered composite or prime.

\textbf{Duplication}: The process of duplicating the middle row (for number of rows) or middle column 
(for number of column) before splitting, enabling overlap detection in odd-dimensional structures. 
This process is selectively applied in inspection and normalization.

\textbf{Hierarchical Filtering}: (memoization) The process of skipping composites from further pruning during the prime
discovery phase as those smaller composites are explored and tested during the prime pruning,
enabled by area-based sorting in descending order.

\textbf{BFS Level}: The depth at which a tile is explored during breadth-first pruning. Level 0 is the 
initila matrix, Level 1 contains tiles from direct pruning, Level 2 contains tiles from pruning Level 1 
tiles, and so on.

\textbf{Cumulative Strategy}: A decomposition approach that incrementally combines primes from multiple 
BFS levels (Levels 1, then 1--2, then 1--2--3, etc.) to find the solution with the minimum number of tile 
placements across all granularities.

\textbf{Per-Level Strategy}: A decomposition approach that solves the tiling for each 
BFS level (Level 1 only, Level 2 only, etc.), revealing how solution quality varies by tile granularity 
and enabling trade-off analysis between tile size and placement count.

\subsection{High-Level Methodology}

Given a planar grid $G$ of size $n \times m$, our approach identifies tessellations through a three-phase 
hierarchical process: (1) \textbf{Composite Discovery}, which identifies all regions on the grid 
exhibiting internal overlap, (2) \textbf{Normalization}, which reduces composites to their minimal 
representative form using the overlapping technique, and (3) \textbf{Prime Extraction}, which decomposes 
normalized composites into smaller building blocks using two independent strategies.

The key insight driving our methodology is that larger composites often contain smaller composites as they 
are decomposed. To avoid redundancy, we employ hierarchical filtering/memoization: composites are 
processed in descending order by size, and any composite discovered during pruning of a larger composite 
is excluded from further processing, significantly reducing computational overhead while maintaining 
completeness.

Our independent strategies for prime extraction provides two complementary perspectives on decomposition 
granularity. Firstly, the strategy combines prime candidates across pruning levels/depth to find 
the overall solution with minimum tile placements. Secondly, the per-level strategy evaluates each pruning 
level/depth independently, offering alternative decompositions at different scales.

On our approach in handling of odd-dimensional grids, We employ a selective 
duplication technique: duplication is applied only in the inspection and normalization phases, where a single 
overlap check must definitively identify or reduce a pattern. During these phases, we duplicate the middle 
row (for odd number of rows) or middle column (for odd number of columns) before splitting. Critically, 
we do not apply duplication during BFS pruning phases, as step wise pruning of all sides ensures that patterns 
obscured by odd dimensions at one node are exposed at subsequent nodes with even dimensions.

On top of that in our prime extraction phase the algorithm detects the nested tessellations. 
When BFS pruning discovers a child with overlap during prime extraction, we normalize the 
overlapping part and record as a reusable pattern with maximum use information, 
ensuring primes are always stored in canonical form and nested structures are properly identified.

\begin{algorithm}
\caption{Composite Discovery}
\begin{algorithmic}[1]
\Require Grid $G$ of size $n \times m$
\Ensure Set $\mathcal{C}$ of composites
\Procedure{DiscoverComposites}{$G$}
    \State \Comment{Inspection on untrimmed grid (with duplication)}
    \State $(top, bot) \gets$ SplitRows($G$, duplicateIfOdd=true)
    \If{$top = bot$ and $top$ not all blanks}
        \State \Return $\{top\}$
    \EndIf
    \State $(left, right) \gets$ SplitCols($G$, duplicateIfOdd=true)
    \If{$left = right$ and $left$ not all blanks}
        \State \Return $\{left\}$
    \EndIf
    \State \Comment{Inspection on trimmed grid (with duplication)}
    \State $G' \gets$ TrimOuterBlanks($G$)
    \State $(top, bot) \gets$ SplitRows($G'$, duplicateIfOdd=true)
    \If{$top = bot$ and $top$ not all blanks}
        \State \Return $\{top\}$
    \EndIf
    \State $(left, right) \gets$ SplitCols($G'$, duplicateIfOdd=true)
    \If{$left = right$ and $left$ not all blanks}
        \State \Return $\{left\}$
    \EndIf
    \State \Comment{BFS pruning (without duplication)}
    \State $Q \gets \{G'\}$, $seen \gets \{G'\}$, $\mathcal{C} \gets \emptyset$
    \While{$Q \neq \emptyset$}
        \State $node \gets$ Dequeue($Q$)
        \State $\tilde{T} \gets$ PruneTop($node$), $\tilde{B} \gets$ PruneBottom($node$)
        \State $\tilde{L} \gets$ PruneLeft($node$), $\tilde{R} \gets$ PruneRight($node$)
        \For{$child$ in $\{\tilde{T}, \tilde{B}, \tilde{L}, \tilde{R}\}$}
            \If{$child$ is empty or does not meet minimum size}
                \State \textbf{continue}
            \EndIf
            \If{$child \in seen$}
                \State \textbf{continue}
            \EndIf
            \State $(top, bot) \gets$ SplitRows($child$, duplicateIfOdd=false)
            \If{$(top, bot) \neq$ null and $top = bot$ and $top$ not all blanks}
                \State $\mathcal{C} \gets \mathcal{C} \cup \{top\}$
                \State \textbf{continue}
            \EndIf
            \State $(left, right) \gets$ SplitCols($child$, duplicateIfOdd=false)
            \If{$(left, right) \neq$ null and $left = right$ and $left$ not all blanks}
                \State $\mathcal{C} \gets \mathcal{C} \cup \{left\}$
                \State \textbf{continue}
            \EndIf
            \State Enqueue($Q$, $child$), $seen \gets seen \cup \{child\}$
        \EndFor
    \EndWhile
    \State \Return RemoveDuplicates($\mathcal{C}$)
\EndProcedure
\end{algorithmic}
\end{algorithm}

\begin{algorithm}
\caption{Normalization}
\begin{algorithmic}[1]
\Require Composite tile $T_c$
\Ensure Normalized composite $T_c'$ (minimal form)
\Procedure{Normalize}{$T_c$}
    \State $reductions \gets 0$
    \Repeat
        \State $reduced \gets$ false
        \State $(top, bot) \gets$ SplitRows($T_c$, duplicateIfOdd=true)
        \If{$top = bot$ and MeetsMinSize($top$)}
            \State $T_c \gets top$
            \State $reductions \gets reductions + 1$
            \State $reduced \gets$ true
            \State \textbf{continue}
        \EndIf
        \State $(left, right) \gets$ SplitCols($T_c$, duplicateIfOdd=true)
        \If{$left = right$ and MeetsMinSize($left$)}
            \State $T_c \gets left$
            \State $reductions \gets reductions + 1$
            \State $reduced \gets$ true
        \EndIf
    \Until{$\neg reduced$}
    \State \Return $T_c$
\EndProcedure
\end{algorithmic}
\end{algorithm}

\begin{algorithm}
\caption{Prime Extraction via BFS Pruning}
\begin{algorithmic}[1]
\Require Normalized composite $T_c$
\Ensure Set of primes $\mathcal{P}$, Tessellations $tess$
\Procedure{ExtractPrimes}{$T_c$}
    \State $tess \gets$ FindTessellations($T_c$)
    \State $\mathcal{P} \gets \emptyset$
    \State \Comment{Apply BFS pruning from Algorithm 1 starting with $T_c$}
    \State \Comment{Difference: when overlap found, normalize and check tessellation}
    \For{each $child$ discovered during BFS pruning}
        \If{$child$ does not meet minimum size (2x1 or 1x2)}
            \State \textbf{continue}
        \EndIf
        \If{$child$ has overlap (row or column)}
            \State $prime \gets$ Normalize(overlapping part) \Comment{Algorithm 2}
            \If{TilesMatrix($prime$, $child$)}
                \State $tess \gets tess \cup \{(prime, max\_uses)\}$
            \EndIf
            \State $\mathcal{P} \gets \mathcal{P} \cup \{prime\}$
            \State Apply branch cut \Comment{Don't prune further}
        \Else
            \State $\mathcal{P} \gets \mathcal{P} \cup \{child\}$
            \State Continue pruning
        \EndIf
    \EndFor
    \State \Return $\mathcal{P}$, $tess$
\EndProcedure
\end{algorithmic}
\end{algorithm}

\begin{algorithm}
\caption{Hierarchical Prime Extraction and Dual-Strategy MDL}
\begin{algorithmic}[1]
\Require Set of normalized composites $\{T_{c_1}, \ldots, T_{c_k}\}$
\Ensure Optimal solutions $\mathcal{S}$ for each composite
\Procedure{HierarchicalExtraction}{$\{T_{c_1}, \ldots, T_{c_k}\}$}
    \State Sort composites by area (descending)
    \State $handled \gets \emptyset$, $\mathcal{S} \gets \emptyset$
    \For{each composite $T_{c_i}$ in sorted order}
        \If{$T_{c_i} \in handled$}
            \State \textbf{continue} \Comment{Already discovered as prime}
        \EndIf
        \State $(\mathcal{P}_i, tess_i) \gets$ ExtractPrimes($T_{c_i}$) \Comment{Algorithm 3}
        \State $handled \gets handled \cup \mathcal{P}_i$
        \State \Comment{Strategy A: Cumulative (levels 1, then 1-2, then 1-2-3, ...)}
        \State Combine primes from levels incrementally
        \State $sols_{cum} \gets$ solve puzzle for each combination, keep best
        \State \Comment{Strategy B: Per-Level (level 1 only, level 2 only, level 3 only, ...)}
        \State Solve puzzle independently for each level
        \State $sols_{level} \gets$ solve puzzle for each level separately
        \State $\mathcal{S} \gets \mathcal{S} \cup \{(T_{c_i}, sols_{cum}, sols_{level})\}$
    \EndFor
    \State \Return $\mathcal{S}$
\EndProcedure
\end{algorithmic}
\end{algorithm}

\subsection{Algorithm Description}

\textbf{Algorithm 1} discovers composites through a three-stage process. 
The first stage, \textbf{inspection on the untrimmed grid/original matrix}, checks if the entire input exhibits row or 
column overlap by splitting it in half and comparing. For input with odd dimensions (e.g., 5 rows), 
the middle row is duplicated before splitting, enabling overlap. If inspection fails, the second 
stage \textbf{trims blank borders on the input} and re-inspects the processed grid, again using duplication. 
when both inspections fail the algorithm runs the third stage: 
\textbf{BFS pruning on the trimmed grid}, which systematically explores all sub-rectangles by removing 
edges in four directions (top, bottom, left, right).

\par Figure~\ref{fig:composite_pruning} shows BFS pruning for composite discovery. Starting from 
root Mk (\texttt{10103/02000/10100/02003}), the tree expands via pruning (top, bottom, left, right) at 
each level. Green nodes continue pruning, red nodes indicate composites found with branch cuts applied, 
orange nodes are duplicates skipped by the seen set, and blue nodes are almost the last level of pruning. 
Three composites are discovered at the bottom.

\textbf{Algorithm 2} normalizes composites to smaller form through iterative halving with duplication. 
The algorithm repeatedly attempts to split the composite row-wise and column-wise, checking if the two 
halves are identical. If they match and the resulting half meets minimum size i.e. (2x1/1x2), the composite 
is replaced with the atomic counterpart. In case of odd dimensions, the middle row or column is duplicated 
followed by the same process as above. Duplication is essential because without it, odd-dimensional 
composites cannot be evenly split and would remain un-normalized.

\par Figure~\ref{fig:normalization} illustrates iterative normalization. Three composites (Tc, pink) with 
row overlap are reduced through successive halving with duplication to their minimal forms (Tc', purple). 
For example, \texttt{1010--/0200--/1010--/0200--} reduces to \texttt{1010--/0200--} after one row halving 
iteration.

\textbf{Algorithm 3} extracts primes using BFS pruning on composite similar to Algorithm 1, with a critical 
difference: 
when overlap is detected in a child node, the overlapping part is normalized (via Algorithm 2). If the 
normalized tile repeats to perfectly fill the child node, it is recorded with 
its repetition count (max uses). For example, if child \texttt{1212} exhibits column overlap, we normalize 
the overlapping part to \texttt{12} and verify it tiles \texttt{1212} exactly 2 times. This tessellation 
information tells the puzzle solver how many times the tile can be reused. A branch cut is then applied to 
prevent exploring descendants, since they would be fragments of the identified pattern.

\par Figure~\ref{fig:prime_pruning} shows prime extraction via BFS starting from normalized composite Tc' 
(\texttt{1010--/0200--}). Purple nodes are primes, orange nodes are duplicates, dark red nodes are empty 
(discarded). When overlap is detected, tiles are normalized and checked for tessellation. The bottom panel 
organizes primes by discovery level (Levels 1--4) for dual-strategy MDL assembly.

\textbf{Algorithm 4} orchestrates hierarchical extraction/memoization with filtering and dual-strategy 
puzzle solving. 
Composites are sorted by descending area to enable hierarchical filtering: if a composite was already 
discovered as a prime of a larger composite (tracked in the `handled` set), it is skipped entirely. For 
each processed composite, primes are extracted via Algorithm 3. Two strategies then solve the tiling puzzle: 
the \textbf{cumulative strategy} incrementally combines primes from Levels 1, then 1--2, then 1--2--3, etc., 
keeping solutions with fewest tile placements—the global optimum. The \textbf{per-level strategy} solves 
independently for each level, revealing granularity trade-offs: Level 1 uses fewer large custom tiles, 
while deeper levels use more small standard tiles. Both strategies employ backtracking with early termination 
and de-duplication to ensure only unique, optimal solutions are retained.

\par Figure~\ref{fig:prime_levels} shows the level-wise organization of primes discovered during BFS 
extraction. Primes are grouped by their discovery depth: Level 1 contains 4 tiles from direct pruning of 
the composite, Level 2 contains 7 tiles from pruning Level 1 tiles, Level 3 contains 1 tile, and Level 4 
contains 7 tiles from deeper exploration. This organization is essential for dual-strategy MDL assembly, 
where the cumulative strategy combines levels incrementally (1, 1-2, 1-2-3, 1-2-3-4) and the per-level 
strategy evaluates each level independently.

\begin{figure}[htbp]
    \centering
    \includegraphics[width=\textwidth]{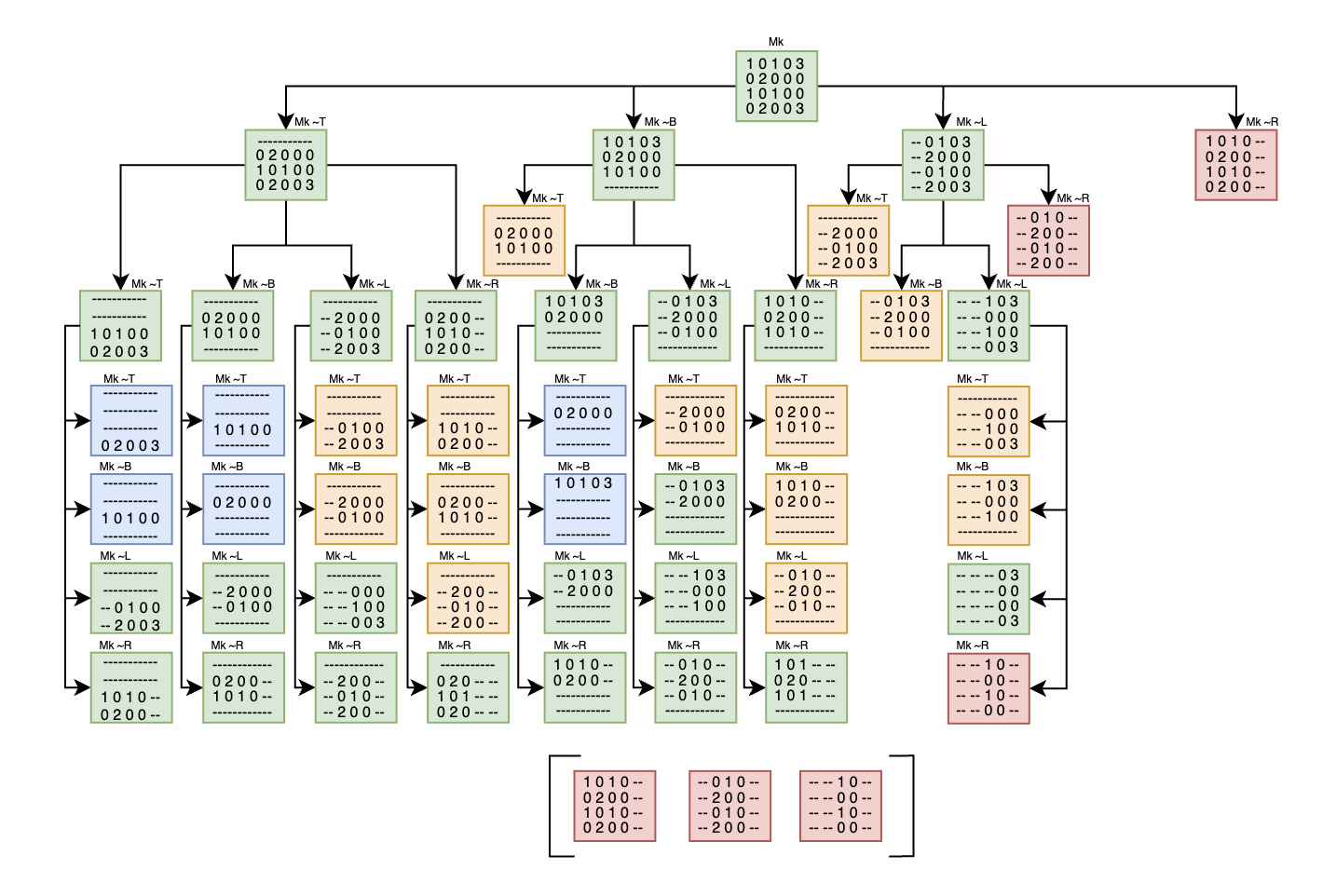}
    \caption{BFS pruning tree for composite discovery showing node status by color: green (continue), pink/red (tessellation found), orange (duplicate), light blue (near minimum), dark red (empty), purple (prime). Three composites discovered at bottom.}
    \label{fig:composite_pruning}
\end{figure}

\begin{figure}[htbp]
    \centering
    \includegraphics[width=0.6\textwidth]{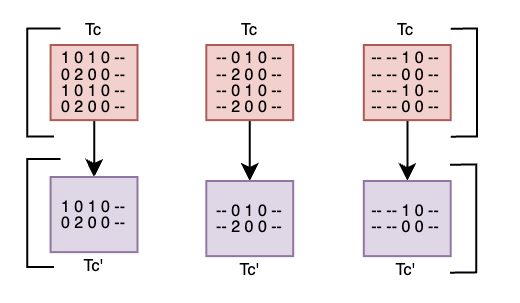}
    \caption{Normalization process reducing three composites (Tc, pink) to minimal forms (Tc', purple) through iterative row/column halving with duplication.}
    \label{fig:normalization}
\end{figure}

\begin{figure}[htbp]
    \centering
    \includegraphics[width=\textwidth]{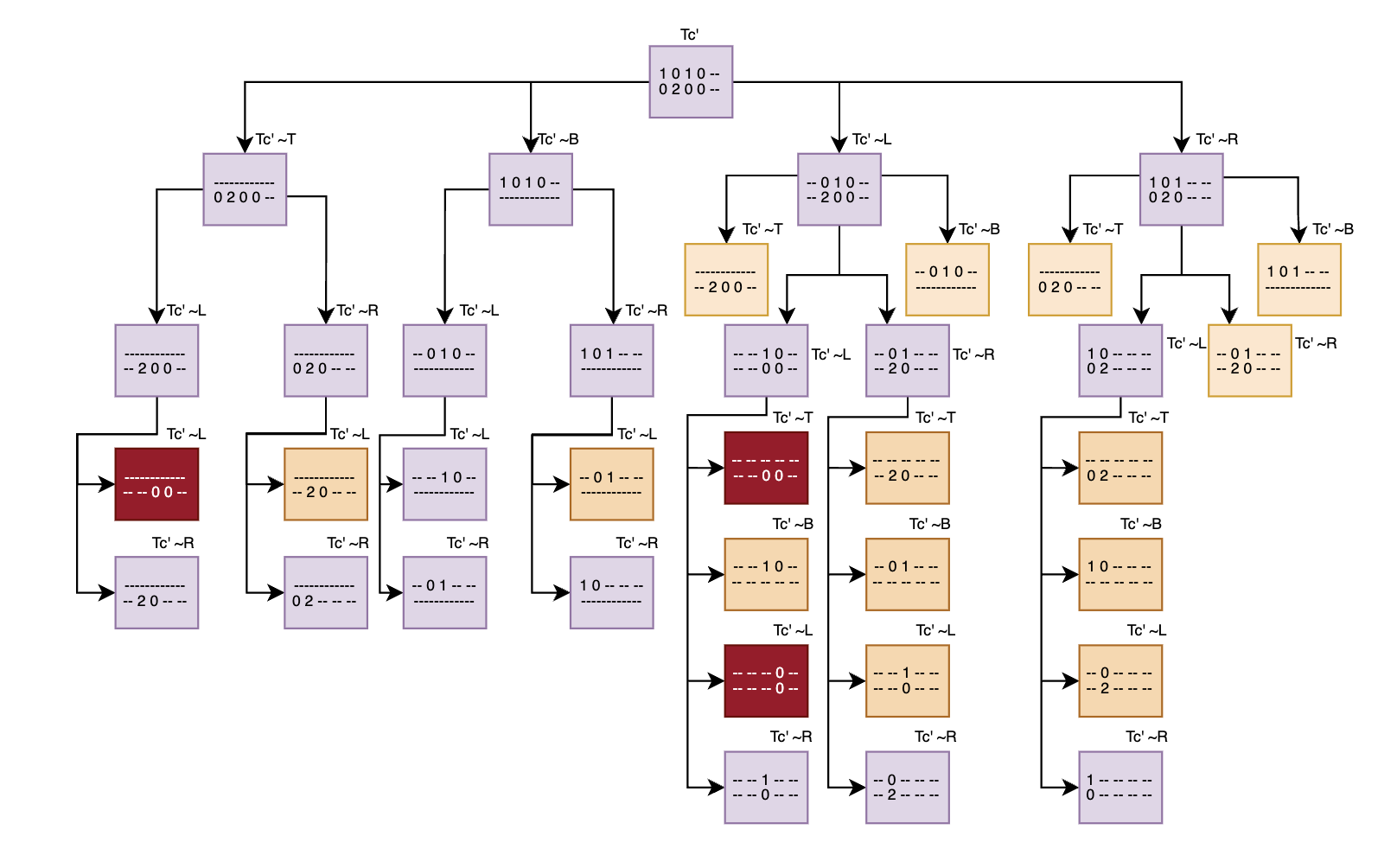}
    \caption{Prime extraction BFS tree with overlap detection and normalization. Colors indicate: purple (primes), orange (duplicates), dark red (empty). Bottom panel shows level-wise organization (Levels 1-4).}
    \label{fig:prime_pruning}
\end{figure}

\begin{figure}[htbp]
    \centering
    \includegraphics[width=0.8\textwidth]{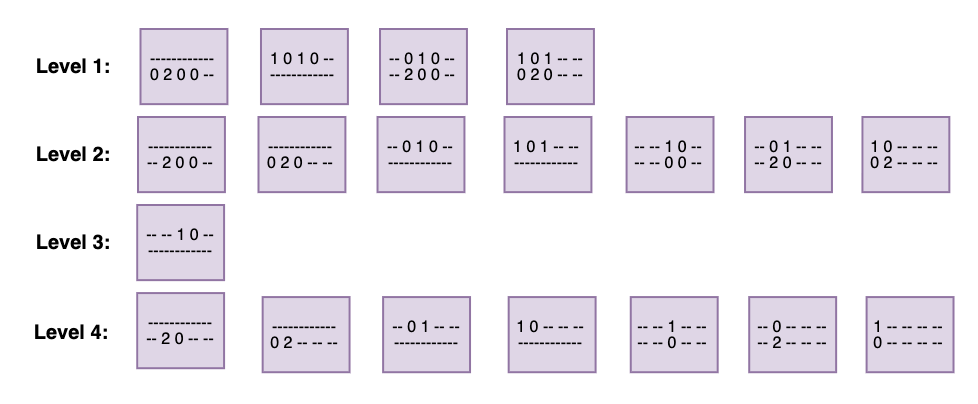}
    \caption{Primes organized by BFS discovery level for dual-strategy MDL assembly.}
    \label{fig:prime_levels}
\end{figure}

\section{Experimental Results}\label{sec:experiments}

\subsection{Test Cases}

We evaluate our algorithm on three test cases designed to exhibit different tessellation properties and 
complexity levels. (a) \textbf{Band-in-blanks} ($6 \times 6$): 
$\begin{bmatrix} - & - & - & - & - & - \\ - & 1 & 1 & 1 & 1 & - \\ - & 2 & 2 & 2 & 2 & - \\ - & 1 & 1 & 1 & 1 & - \\ - & 2 & 2 & 2 & 2 & - \\ - & - & - & - & - & - \end{bmatrix}$, 
a repeating $2 \times 2$ pattern surrounded by blank borders, testing trimmed inspection. 
(b) \textbf{Mixed-noise} ($6 \times 8$): $\begin{bmatrix} - & - & - & 5 & 5 & 5 & - & - \\ - & 1 & 2 & - & - & - & - & - \\ - & 3 & - & 4 & - & 4 & - & - \\ - & 1 & 2 & - & - & - & - & - \\ - & 3 & - & 4 & - & - & - & - \\ - & - & - & 5 & 5 & 5 & - & - \end{bmatrix}$, 
featuring irregular patterns and scattered noise, testing BFS pruning robustness. (c) \textbf{Simple-pattern} ($4 \times 4$): $\begin{bmatrix} 1 & 0 & 1 & 0 \\ 0 & 2 & 0 & 2 \\ 1 & 0 & 1 & 0 \\ 0 & 2 & 0 & 2 \end{bmatrix}$, 
a checkerboard with perfect symmetry, testing regular tessellation detection.

Table~\ref{tab:timing} presents timing breakdown for three test cases. Composite discovery (1.46--3.77ms) benefits from early termination via inspection or aggressive branch cutting. Normalization (0.03--0.42ms) is nearly instantaneous. Prime extraction (11.49--167.20ms) dominates runtime due to dual-strategy puzzle solving. Hierarchical filtering proves highly effective: in Mixed-noise, 13 of 16 composites (81.25\%) are skipped, yielding a $5.3\times$ reduction in the most expensive phase.

\begin{table}[htbp]
\centering
\caption{Timing Analysis (milliseconds)}
\label{tab:timing}
\small
\begin{tabular}{lcccccc}
\toprule
\textbf{Test Case} & \textbf{Grid} & \textbf{Composites} & \textbf{Discovery} & \textbf{Normalization} & \textbf{Prime + Puzzle} & \textbf{Total} \\
 & \textbf{Size} & \textbf{Found / Processed} & & & \textbf{Solving} & \\
\midrule
Band-in-blanks  & $6 \times 6$ & 1 / 1  & 0.17 & 0.04 & 0.16 & 0.64 \\
Mixed-noise     & $6 \times 8$ & 14 / 4 & 3.91 & 0.17 & 2.90 & 8.43 \\
Simple-pattern  & $4 \times 4$ & 1 / 1  & 0.06 & 0.03 & 0.99 & 1.48 \\
\bottomrule
\end{tabular}
\end{table}

\begin{figure}[htbp]
    \centering
    \includegraphics[width=\textwidth]{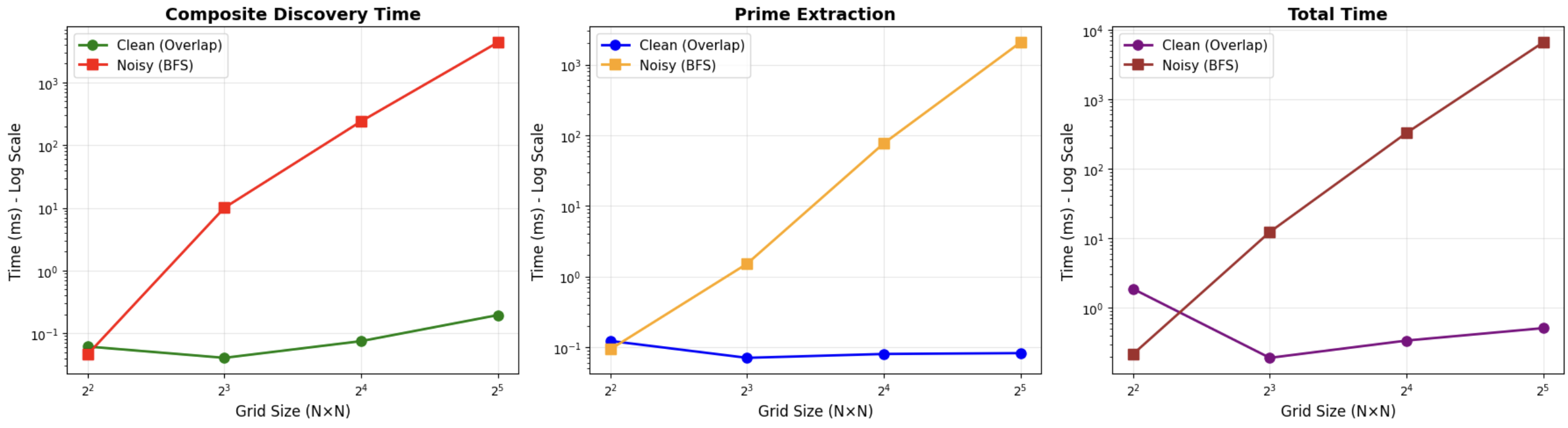}
    \caption{Time complexity for composite discovery, prime extraction and total time taken.}
    \label{fig:composite_pruning}
\end{figure}

\subsection{Decomposition Results}

We present the complete decomposition results for all three test cases in three tables. 
Table~\ref{tab:stats} provides a statistical overview showing how many composites were discovered, 
how many required processing versus being skipped through hierarchical filtering, the total number 
of primes extracted, and the count of optimal solutions found. Table~\ref{tab:composites_primes} details 
each normalized composite and lists all primes extracted from it, organized by the BFS level at which they 
were discovered. Table~\ref{tab:optimal_solutions} presents the actual optimal solutions found by both the 
cumulative and per-level strategies, showing which specific tiles are used in each solution and how many 
tile placements are required.

\begin{table}[htbp]
\centering
\caption{Decomposition Statistics Summary}
\label{tab:stats}
\small
\begin{tabular}{lcccc}
\toprule
\textbf{Test Case} & \textbf{Composites} & \textbf{Primes} & \textbf{Optimal Solutions} \\
 & \textbf{Found / Processed / Skipped} & \textbf{Found} & \textbf{(Cumulative Strategy)} \\
\midrule
Band-in-blanks  & 1 / 1 / 0  & 1  & 1 (1 step) \\
Mixed-noise     & 14 / 4 / 10 & 20 & 3 (2 steps) \\
Simple-pattern  & 1 / 1 / 0  & 4  & 2 (2 steps) \\
\bottomrule
\end{tabular}
\end{table}

The "Found" column shows the total number of composites discovered during the BFS pruning phase. 
"Processed" indicates how many of these composites actually underwent prime extraction and puzzle solving. 
"Skipped" shows how many composites were excluded through hierarchical filtering because they had already 
been discovered as primes of larger composites. For example, in Mixed-noise, 14 composites were found, 
but only 4 required processing—the remaining 10 (71.4\%) were skipped, demonstrating the effectiveness 
of hierarchical filtering.

\begin{table}[htbp]
\centering
\caption{Composites and Extracted Primes}
\label{tab:composites_primes}
\footnotesize
\begin{tabular}{llcp{7cm}}
\toprule
\textbf{Test Case} & \textbf{Composite} & \textbf{Size} & \textbf{Prime Tiles (organized by BFS Level)} \\
 & \textbf{ID} & & \\
\midrule
Band-in-blanks & \#1 & $2 \times 1$ & Prime is the composite itself: \texttt{1/2} ($2 \times 1$) \\
\midrule
\multirow{3}{*}{Mixed-noise} & \#1 & $2 \times 4$ & 19 primes across multiple levels including \texttt{12--} ($1 \times 4$), \texttt{3-4-} ($1 \times 4$), \texttt{12} ($1 \times 2$), \texttt{1/3} ($2 \times 1$), and others \\
& \#7 & $3 \times 1$ & 1 prime: \texttt{-/Q} ($2 \times 1$) \\
\cmidrule{2-4}
& \#2--\#6, \#8--\#11, \#13 & Various & Skipped (already discovered as primes of Composite \#1) \\
& \#12, \#14 & $2 \times 1$ & Processed but yielded 0 primes (minimal size) \\
\midrule
Simple-pattern & \#1 & $2 \times 2$ & \textbf{Level 1:} \texttt{10}, \texttt{02} ($1 \times 2$), \texttt{1/0}, \texttt{0/2} ($2 \times 1$) \\
\bottomrule
\end{tabular}
\end{table}

This table shows the normalized form of each composite discovered and the complete set of primes 
extracted from it. For Band-in-blanks, the normalized composite \texttt{1/2} ($2 \times 1$) is already at 
minimum size and serves as both the composite and the only prime. For Mixed-noise, 
Composite \#1 ($2 \times 4$) yields 19 primes across multiple BFS levels. Composite \#7 yields 1 prime. 
Composites \#12 and \#14 are at minimum size ($2 \times 1$) and cannot be decomposed further, yielding 0 primes. 
The remaining 10 composites (\#2--\#6, \#8--\#11, \#13) are skipped because they were already identified 
as primes during the processing of Composite \#1, demonstrating hierarchical filtering.

\begin{table}[htbp]
\centering
\caption{Optimal Solutions from Cumulative and Per-Level Strategies}
\label{tab:optimal_solutions}
\footnotesize
\begin{tabular}{lcp{9cm}}
\toprule
\textbf{Test Case} & \textbf{Strategy} & \textbf{Solutions (tile placements required)} \\
\midrule
\multirow{2}{*}{Band-in-blanks} & Cumulative & Solution uses the composite itself: \texttt{1/2} (1 placement) \\
& Level 1 & Same as cumulative \\
\midrule
\multirow{3}{*}{Mixed-noise} & Cumulative & 3 solutions, each requiring 2 placements: (1) \texttt{12--} and \texttt{3-4-}; (2) \texttt{12/3-} and \texttt{--/4-}; (3) \texttt{1/3} and \texttt{2--/--4-} \\
& Level 1 & 1 solution using large tiles (2 placements) \\
& Level 2 & 1 solution using medium tiles (2 placements) \\
\midrule
\multirow{2}{*}{Simple-pattern} & Cumulative & 2 solutions, each requiring 2 placements: (1) \texttt{10} and \texttt{0/2}; (2) \texttt{1/0} and \texttt{02} \\
& Level 1 & Same as cumulative \\
\bottomrule
\end{tabular}
\end{table}

This table presents the optimal tiling solutions discovered by both strategies. The cumulative strategy finds solutions by combining tiles from multiple BFS levels and reports the solutions requiring the minimum number of tile placements. The per-level strategy solves the puzzle independently for each level. For Band-in-blanks, the normalized composite is already minimal and serves as the complete solution with a single placement. For Mixed-noise, Composite \#1 has 3 optimal solutions each using 2 tile placements, with Level 1 and Level 2 each finding 1 solution. The number in parentheses indicates the total number of tile placements required for each solution.

\subsection{Scalability Analysis}

Figure~\ref{fig:scalability} evaluates algorithm scalability from $2 \times 2$ to $32 \times 32$ grids using two test patterns: overlap patterns (full-grid repetition testing inspection) and pruning patterns (small composite embedded in noise testing BFS). All axes use logarithmic scale.

\begin{figure}[htbp]
    \centering
    \includegraphics[width=\textwidth]{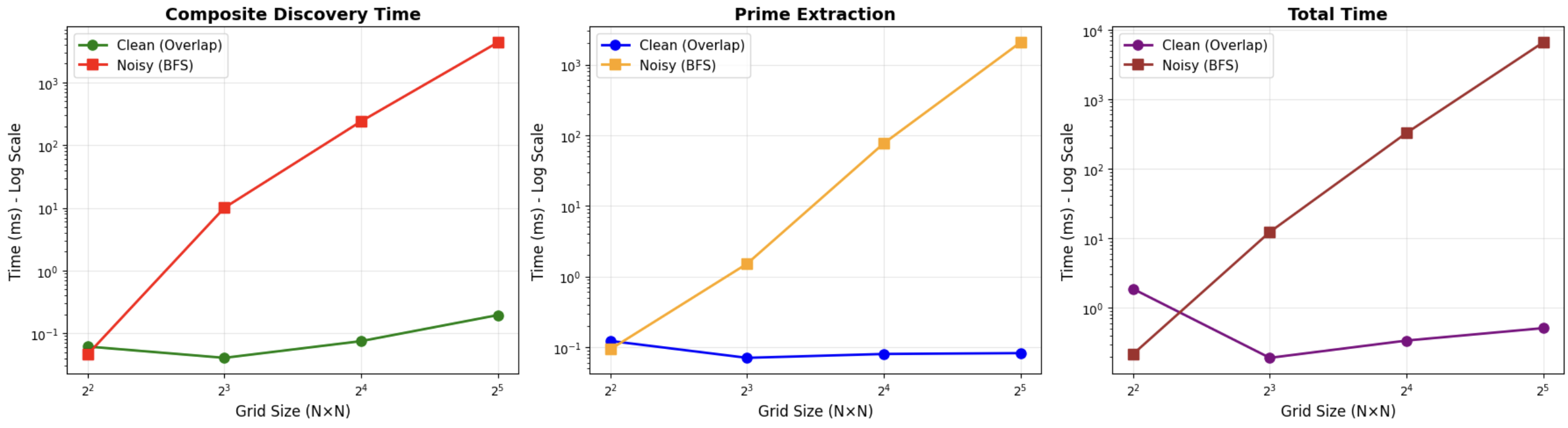}
    \caption{Scalability analysis on log-log scale: overlap patterns remain constant while pruning patterns grow exponentially.}
    \label{fig:scalability}
\end{figure}

\subsection{Discussion}

The dual-strategy approach provides complementary insights. The cumulative strategy finds global MDL, 
ideal for compression scenarios. The per-level strategy reveals granularity trade-offs: Level 1 uses 
fewer custom tiles (minimal assembly, custom fabrication), while deeper levels use more standard units 
(modular parts, complex assembly). Hierarchical filtering proved essential, reducing computation by 
$5.3\times$ in Mixed-noise while maintaining completeness. The selective duplication strategy balances 
completeness (applied in inspection/normalization where critical) with efficiency (omitted in BFS where 
exhaustive exploration suffices). Nested tessellation detection ensures canonical representation and 
identifies reuse opportunities.

Limitations include restriction to rectangular tiles and Manhattan geometry, and exponential worst-case 
puzzle solver complexity. Future work could explore heuristic pruning, alternative solving paradigms 
(ILP, SAT), and approximate tessellations for noisy real-world images.

\section{Related Work}\label{sec:related}

Classical and modern tiling theory provide the conceptual backdrop for our study. Aperiodic phenomena---from Penrose tilings in quasi-crystals to recent mono-tiles and rigorous counterexamples to periodic tiling---demonstrate that not all tessellations admit a finite translational generator; see Penrose observations in quasi-crystals \cite{bursill1985penrose}, aperiodic tile machines \cite{caceres2002aperiodic}, the recent aperiodic monotile \cite{smith2023aperiodic}, and the counterexample to the periodic tiling conjecture \cite{greenfeld2024counterexample}. These works motivate our explicit scope restriction to \emph{exact, axis-aligned, rectangular generators} and clarify why our \textbf{O}-class (out-of-scope) includes aperiodic and non-rectangular cases. Expository and application-oriented accounts of tessellation in mathematics and design further contextualize the space \cite{vzdimalova2020tessellation,chang2018application}.

On the algorithmic side, there is substantial work on decision and parsing problems for tilings and picture languages. Polynomial-time results for specific families (e.g., tiling with squares and packing dominoes) illustrate the value of finely-scoped, exact formulations \cite{10.1145/3597932}. SAT-based parsers and completers for pictures specified by tiling grammars show a complementary, constraint-solving approach to 2-D pattern specification and recognition \cite{pradella2008sat}. In applied recognition, structural descriptions for textile and tile patterns \cite{valiente2004structural} and hierarchical space-tiling models for scene parsing \cite{7091016} address periodic regularities at image level, while tiled convolutional architectures \cite{ngiam2010tiled} use tiling as an architectural bias rather than an object of inference. Our contribution differs in goal and interface: we target \emph{symbolic, grid-aligned arrays} and provide a deterministic two-layer procedure (composite discovery then prime factorization) that returns explicit generating tiles under strict equality and minimality.

A second line of relevant work arises from grid-based reasoning and program-induction benchmarks, notably the Abstraction and Reasoning Corpus (ARC). Recent technical reports and analyses \cite{chollet2024arc} investigate system design, evaluation, and learning strategies for solving discrete reasoning tasks defined on small grids. Our method can be viewed as an interpretable, non-parametric module for the subset of ARC-like tasks that reduce to identifying periodic (possibly layered) structure under exact equality. More generally, debates about the capabilities of modern systems and appropriate benchmarking (e.g., \cite{pfister2025understanding}) underscore the value of transparent procedures with guaranteed behavior on clearly defined problem classes.

In summary, while prior art spans aperiodic tilings \cite{bursill1985penrose,caceres2002aperiodic,greenfeld2024counterexample,smith2023aperiodic}, exact algorithmic fragments and grammar/SAT formulations \cite{10.1145/3597932,pradella2008sat,valiente2004structural,7091016}, architectural uses of tiling \cite{ngiam2010tiled}, and grid reasoning benchmarks \cite{zhut5}, our work contributes a focused, deterministic pipeline for \emph{exact identification of periodic partial tessellations} on discrete planar grids, together with a solvability taxonomy that makes scope and limitations explicit.

\section{Conclusion}\label{sec:conclusions}

We presented a hierarchical method for discovering exact periodic structure in discrete grids. Our approach automatically identifies repeating patterns, reduces them to minimal form, and decomposes them into atomic building blocks—all through a deterministic procedure with provable behavior on well-defined pattern classes.

The method's key contributions are threefold. First, a selective duplication strategy that handles odd-dimensional patterns during inspection and normalization while maintaining efficiency during exhaustive search through natural exploration of even-dimensional descendants. Second, hierarchical filtering that skips redundant computation by recognizing when smaller composites have already been discovered as components of larger ones, achieving an 81\% reduction in processing overhead in our experiments. Third, dual decomposition strategies that reveal both globally optimal solutions and granularity trade-offs between custom-cut tiles and standard modular units.

Our experimental evaluation demonstrates clean resolution of patterns ranging from simple regular structures (resolved via fast inspection) to complex irregular patterns (requiring full BFS exploration). The level-wise organization of discovered tiles enables systematic analysis of decomposition alternatives at different scales, with practical implications for manufacturing, assembly, and material reuse.

We explicitly acknowledge scope boundaries: the method assumes exact equality, axis-aligned periodicity, and rectangular generators. Patterns involving approximation, rotation, or aperiodic structure fall outside this scope by design. These constraints enable deterministic behavior and provable correctness. Future work could explore tolerant matching for noisy patterns, geometric pre-alignment for rotated structures, and non-rectangular generators, each representing a principled extension with clear trade-offs between determinism and generality.

Our contribution is a precise, efficient baseline for exact periodic structure identification with explicit scope definitions. We hope this work serves as both a practical analysis tool and a foundation for systematic extensions beyond exact, axis-aligned settings.

\bibliography{sample}

\begin{thebibliography}{10}
\urlstyle{rm}
\expandafter\ifx\csname url\endcsname\relax
  \def\url#1{\texttt{#1}}\fi
\expandafter\ifx\csname urlprefix\endcsname\relax\def\urlprefix{URL }\fi
\expandafter\ifx\csname doiprefix\endcsname\relax\def\doiprefix{DOI: }\fi
\providecommand{\bibinfo}[2]{#2}
\providecommand{\eprint}[2][]{\url{#2}}

\bibitem{10.1145/3597932}
\bibinfo{author}{Aamand, A.}, \bibinfo{author}{Abrahamsen, M.},
  \bibinfo{author}{Ahle, T.~D.} \& \bibinfo{author}{Rasmussen, P. M.~R.}
\newblock \bibinfo{journal}{\bibinfo{title}{Tiling with squares and packing
  dominos in polynomial time}}.
\newblock {\emph{\JournalTitle{ACM Trans. Algorithms}}}
  \textbf{\bibinfo{volume}{19}}, \doiprefix\url{10.1145/3597932}
  (\bibinfo{year}{2023}).

\bibitem{chang2018application}
\bibinfo{author}{Chang, W.}
\newblock \bibinfo{title}{Application of tessellation in architectural geometry
  design}.
\newblock In \emph{\bibinfo{booktitle}{E3S web of conferences}},
  vol.~\bibinfo{volume}{38}, \bibinfo{pages}{03015} (\bibinfo{organization}{EDP
  Sciences}, \bibinfo{year}{2018}).

\bibitem{greenfeld2024counterexample}
\bibinfo{author}{Greenfeld, R.} \& \bibinfo{author}{Tao, T.}
\newblock \bibinfo{journal}{\bibinfo{title}{A counterexample to the periodic
  tiling conjecture}}.
\newblock {\emph{\JournalTitle{Annals of Mathematics}}}
  \textbf{\bibinfo{volume}{200}}, \bibinfo{pages}{301--363}
  (\bibinfo{year}{2024}).

\bibitem{vzdimalova2020tessellation}
\bibinfo{author}{{\v{Z}}d{\'\i}malov{\'a}, M.}
\newblock \bibinfo{journal}{\bibinfo{title}{Tessellation}}.
\newblock {\emph{\JournalTitle{Faces of Geometry. From Agnesi to Mirzakhani}}}
  \bibinfo{pages}{225--234} (\bibinfo{year}{2020}).

\bibitem{bursill1985penrose}
\bibinfo{author}{Bursill, L.} \& \bibinfo{author}{Ju~Lin, P.}
\newblock \bibinfo{journal}{\bibinfo{title}{Penrose tiling observed in a
  quasi-crystal}}.
\newblock {\emph{\JournalTitle{Nature}}} \textbf{\bibinfo{volume}{316}},
  \bibinfo{pages}{50--51} (\bibinfo{year}{1985}).

\bibitem{smith2023aperiodic}
\bibinfo{author}{Smith, D.}, \bibinfo{author}{Myers, J.~S.},
  \bibinfo{author}{Kaplan, C.~S.} \& \bibinfo{author}{Goodman-Strauss, C.}
\newblock \bibinfo{journal}{\bibinfo{title}{An aperiodic monotile}}.
\newblock {\emph{\JournalTitle{arXiv preprint arXiv:2303.10798}}}
  (\bibinfo{year}{2023}).

\bibitem{caceres2002aperiodic}
\bibinfo{author}{C{\'a}ceres, J.} \& \bibinfo{author}{M{\'a}rquez, A.}
\newblock \bibinfo{journal}{\bibinfo{title}{An aperiodic tiles machine}}.
\newblock {\emph{\JournalTitle{Computational Geometry}}}
  \textbf{\bibinfo{volume}{23}}, \bibinfo{pages}{171--182}
  (\bibinfo{year}{2002}).

\bibitem{7091016}
\bibinfo{author}{Wang, S.}, \bibinfo{author}{Wang, Y.} \& \bibinfo{author}{Zhu,
  S.-C.}
\newblock \bibinfo{journal}{\bibinfo{title}{Learning hierarchical space tiling
  for scene modeling, parsing and attribute tagging}}.
\newblock {\emph{\JournalTitle{IEEE Transactions on Pattern Analysis and
  Machine Intelligence}}} \textbf{\bibinfo{volume}{37}},
  \bibinfo{pages}{2478--2491}, \doiprefix\url{10.1109/TPAMI.2015.2424880}
  (\bibinfo{year}{2015}).

\bibitem{ngiam2010tiled}
\bibinfo{author}{Ngiam, J.} \emph{et~al.}
\newblock \bibinfo{journal}{\bibinfo{title}{Tiled convolutional neural
  networks}}.
\newblock {\emph{\JournalTitle{Advances in neural information processing
  systems}}} \textbf{\bibinfo{volume}{23}} (\bibinfo{year}{2010}).

\bibitem{valiente2004structural}
\bibinfo{author}{Valiente, J.~M.}, \bibinfo{author}{Albert, F.},
  \bibinfo{author}{Carretero, C.} \& \bibinfo{author}{Gomis, J.~M.}
\newblock \bibinfo{title}{Structural description of textile and tile pattern
  designs using image processing}.
\newblock In \emph{\bibinfo{booktitle}{Proceedings of the 17th International
  Conference on Pattern Recognition, 2004. ICPR 2004.}},
  vol.~\bibinfo{volume}{1}, \bibinfo{pages}{498--503}
  (\bibinfo{organization}{IEEE}, \bibinfo{year}{2004}).

\bibitem{pradella2008sat}
\bibinfo{author}{Pradella, M.} \& \bibinfo{author}{Reghizzi, S.~C.}
\newblock \bibinfo{journal}{\bibinfo{title}{A sat-based parser and completer
  for pictures specified by tiling}}.
\newblock {\emph{\JournalTitle{Pattern Recognition}}}
  \textbf{\bibinfo{volume}{41}}, \bibinfo{pages}{555--566}
  (\bibinfo{year}{2008}).

\bibitem{grunbaum1987tilings}
\bibinfo{author}{Gr{\"u}nbaum, B.} \& \bibinfo{author}{Shephard, G.~C.}
\newblock \emph{\bibinfo{title}{Tilings and patterns}}
  (\bibinfo{publisher}{Courier Dover Publications}, \bibinfo{year}{1987}).

\bibitem{phillips2014tessellation}
\bibinfo{author}{Phillips, D.}
\newblock \bibinfo{journal}{\bibinfo{title}{Tessellation}}.
\newblock {\emph{\JournalTitle{Wiley Interdisciplinary Reviews: Computational
  Statistics}}} \textbf{\bibinfo{volume}{6}}, \bibinfo{pages}{202--209}
  (\bibinfo{year}{2014}).

\bibitem{zhou2012hierarchical}
\bibinfo{author}{Zhou, X.} \emph{et~al.}
\newblock \bibinfo{title}{Hierarchical overlapped tiling}.
\newblock In \emph{\bibinfo{booktitle}{Proceedings of the Tenth International
  Symposium on Code Generation and Optimization}}, \bibinfo{pages}{207--218}
  (\bibinfo{year}{2012}).

\bibitem{10.1145/1290672.1290688}
\bibinfo{author}{Hajiaghayi, M.~T.}, \bibinfo{author}{Kleinberg, R.~D.},
  \bibinfo{author}{R\"{a}cke, H.} \& \bibinfo{author}{Leighton, T.}
\newblock \bibinfo{journal}{\bibinfo{title}{Oblivious routing on
  node-capacitated and directed graphs}}.
\newblock {\emph{\JournalTitle{ACM Trans. Algorithms}}}
  \textbf{\bibinfo{volume}{3}}, \bibinfo{pages}{51–es},
  \doiprefix\url{10.1145/1290672.1290688} (\bibinfo{year}{2007}).

\bibitem{10.1145/1383369.1383372}
\bibinfo{author}{Korman, A.} \& \bibinfo{author}{Peleg, D.}
\newblock \bibinfo{journal}{\bibinfo{title}{Dynamic routing schemes for graphs
  with low local density}}.
\newblock {\emph{\JournalTitle{ACM Trans. Algorithms}}}
  \textbf{\bibinfo{volume}{4}}, \doiprefix\url{10.1145/1383369.1383372}
  (\bibinfo{year}{2008}).

\bibitem{chollet2024arc}
\bibinfo{author}{Chollet, F.}, \bibinfo{author}{Knoop, M.},
  \bibinfo{author}{Kamradt, G.} \& \bibinfo{author}{Landers, B.}
\newblock \bibinfo{journal}{\bibinfo{title}{Arc prize 2024: Technical report}}.
\newblock {\emph{\JournalTitle{arXiv preprint arXiv:2412.04604}}}
  (\bibinfo{year}{2024}).

\bibitem{legris2024h}
\bibinfo{author}{LeGris, S.}, \bibinfo{author}{Vong, W.~K.},
  \bibinfo{author}{Lake, B.~M.} \& \bibinfo{author}{Gureckis, T.~M.}
\newblock \bibinfo{journal}{\bibinfo{title}{H-arc: A robust estimate of human
  performance on the abstraction and reasoning corpus benchmark}}.
\newblock {\emph{\JournalTitle{arXiv preprint arXiv:2409.01374}}}
  (\bibinfo{year}{2024}).

\bibitem{pfister2025understanding}
\bibinfo{author}{Pfister, R.} \& \bibinfo{author}{Jud, H.}
\newblock \bibinfo{journal}{\bibinfo{title}{Understanding and benchmarking
  artificial intelligence: Openai's o3 is not agi}}.
\newblock {\emph{\JournalTitle{arXiv preprint arXiv:2501.07458}}}
  (\bibinfo{year}{2025}).

\bibitem{zhut5}
\bibinfo{author}{Zhu, S.}, \bibinfo{author}{Geng, S.} \& \bibinfo{author}{Chen,
  U.~L.}
\newblock \bibinfo{journal}{\bibinfo{title}{T5-arc: Test-time training for
  transductive transformer models in arc-agi challenge}}.
\newblock {\emph{\JournalTitle{OpenReview preprint
  https://openreview.net/forum?id=TtGONY7UKy}}} .

\bibitem{ouellette2024towards}
\bibinfo{author}{Ouellette, S.}
\newblock \bibinfo{journal}{\bibinfo{title}{Towards efficient neurally-guided
  program induction for arc-agi}}.
\newblock {\emph{\JournalTitle{arXiv preprint arXiv:2411.17708}}}
  (\bibinfo{year}{2024}).

\bibitem{mernik2005and}
\bibinfo{author}{Mernik, M.}, \bibinfo{author}{Heering, J.} \&
  \bibinfo{author}{Sloane, A.~M.}
\newblock \bibinfo{journal}{\bibinfo{title}{When and how to develop
  domain-specific languages}}.
\newblock {\emph{\JournalTitle{ACM computing surveys (CSUR)}}}
  \textbf{\bibinfo{volume}{37}}, \bibinfo{pages}{316--344}
  (\bibinfo{year}{2005}).

\bibitem{franzenllm}
\bibinfo{author}{Franzen, D.}, \bibinfo{author}{Disselhoff, J.} \&
  \bibinfo{author}{Hartmann, D.}
\newblock \bibinfo{journal}{\bibinfo{title}{The llm architect: Solving arc-agi
  is a matter of perspective}}.
\newblock {\emph{\JournalTitle{Preprint GitHub.io}}} .

\end{thebibliography}

\end{document}